\documentclass{article}
\usepackage{amsmath,amssymb,graphicx,amsthm}%,ifthen,epstopdf,fancyhdr,color}
\usepackage{enumitem}
\usepackage{pgfplots}
\usepackage[margin=1in]{geometry}

\usepackage{graphicx} % more modern
\usepackage{subfigure} 

\usepackage{comment}
\usepackage{titling}

\usepackage[numbers,sort&compress]{natbib}

\usepackage{authblk}
\usepackage{pdfpages}

\usetikzlibrary{external} 
\tikzset{external/system call={lualatex \tikzexternalcheckshellescape -halt-on-error -interaction=batchmode -jobname "\image" "\texsource"}}
\usepackage{color}
\usepackage{ifthen}

  \newcommand{\R}{\ensuremath{\mathbb{R}}}

  \newcommand{\Nc}{\mathcal{N}}

\newcommand{\V}[1]{\ensuremath{\mathbf{#1}}}

\newcommand{\norm}[1]{\left|\left| #1 \right|\right|}

\newcommand{\TODO}[1]{ 
\ifx\NOTES\undefined\else
{\color{red} [!]}\footnote{ {\color{red} TODO: #1}}
\fi
}

\newcommand{\NOTE}[1]{ 
\ifx\NOTES\undefined\else
  \footnote{ {\color{blue} NOTE: #1}}  
\fi
}

\newcommand{\ecomment}[1]{ 
\ifx\NOTES\undefined\else 
{\color{blue}[E]}\footnote{ {\color{blue} Eran: #1}}
\fi
}

\newcommand{\mcomment}[1]{ 
\ifx\NOTES\undefined\else
  {\color{green} [M]}\footnote{ {\color{green} Matan: #1}}  
\fi
}

\newcommand{\iid}{\stackrel{\text{iid}}{\sim}}

\newcommand{\Pd}[3]{\ifthenelse{\equal{#3}{1}}{\frac{\partial #1}{\partial #2}}{\frac{\partial^{#3} #1}{\partial #2^{#3}}}}

\def \NOTES

\theoremstyle{theorem} \newtheorem{theorem}{Theorem}
\theoremstyle{definition} 
\theoremstyle{definition} 
\theoremstyle{definition} \newtheorem{lemma}{Lemma}
\theoremstyle{definition} 
\theoremstyle{definition}

\newcommand{\name}{{\em ReFACTor}}
\newcommand{\nameplus}{{\em ReFACTor+}}
\newcommand{\nameS}{{\em ReFACTor$^*$}}
\newcommand{\JL}{{\em JL}}
\newcommand{\JLS}{{\em JL$^*$}}

\newcommand{\rf}{{\ensuremath{RF}}}

\newcommand{\Prob}{\ensuremath{Pr}}
\newcommand{\T}{\ensuremath{{\scriptscriptstyle \top}}}

\title{\name{}: Practical Low-Rank  
Matrix Estimation Under Column-Sparsity}

\author[*,1]{Matan Gavish}%\footnote{Equal contribution}}
\author[*,2]{Regev Schweiger}%\footnotemark[1]}
\author[2]{Elior Rahmani}
\author[3,4]{Eran Halperin}
\affil[*]{Equal contribution}
\affil[1]{School of Computer Science and Engineering, Hebrew University of Jerusalem, Jerusalem, Israel}
\affil[2]{Blavatnik School of Computer Science, Tel Aviv University, Tel Aviv, Israel}
\affil[3]{Department of Computer Science, University of California, Los Angeles, CA, USA}
\affil[4]{Department of Anesthesiology and Perioperative Medicine, University of California, Los Angeles, CA, USA}

\date{}

\begin{document} 

\maketitle

\begin{abstract}

    Various problems in data analysis and statistical genetics call for recovery
    of a column-sparse, low-rank matrix  from noisy observations. We propose 
    \name{}, a
    simple variation of the classical Truncated Singular Value Decomposition
    (TSVD)
    algorithm. 
    %
    %\name{} is shown to outperform TSVD both in theory
    %and in practice, suggesting it should be preferred to the default truncated
    %SVD whenever column sparsity is suspected. 
    %
    In contrast to previous sparse
    principal component analysis (PCA) algorithms,  
    our algorithm can provably reveal a low-rank
    signal matrix better, and often significantly better,
    than the widely used TSVD, making it the algorithm of choice 
    whenever column-sparsity is suspected. 
        Empirically, we observe that \name{} consistently outperforms TSVD even when the underlying signal is not sparse, 
        suggesting that it is generally safe to use \name{} instead of TSVD and PCA. 
    The algorithm is extremely simple
    to implement and its running time is dominated by the runtime of PCA, making it as practical as standard principal component
    analysis.

        %We discuss connections between this problem and sparse PCA.
    %and sparse PCA algorithms.
\end{abstract}
%\tableofcontents

\section{Introduction}

%\TODO{why $m>n$? why not $m<n$?} 

Principal Component Analysis (PCA) or  Truncated Singular Value Decomposition
(TSVD) are arguably among the most ubiquitous methods used for data analysis in
science and engineering
\cite{Alter2000,Cattell1966,Jackson1993,Price2006,Edfors1998}. The main
objective of these methods is to search for low rank signals hidden in a data
matrix. Formally, suppose that $X$ is an unknown low rank $m$-by-$n$ matrix.
We observe a single noisy $m$-by-$n$ matrix $Y$, obeying $Y=X+\sigma
Z$, where $Z$ is an unknown noise matrix and $\sigma>0$ is the noise level. Our goal is to estimate $X$ from the
data $Y$. Using the Singular Value Decomposition (SVD) of $Y$, we can write
\begin{eqnarray} 
  Y=\sum_{i=1}^m y_i \V{u}_i \V{v}^\T_i 
  \label{TSVD:eq}
\end{eqnarray}
where $\V{u}_i\in\R^m$ and $\V{v}_i\in\R^n$, $i=1,\ldots,m$ are
the left and right singular vectors corresponding to the singular value $y_i$.
The TSVD estimator~\cite{Golub1965} is 
\begin{eqnarray} \hat{X}_r = \sum_{i=1}^r
  y_i \V{u}_i \V{v}^\T_i\,, 
  \label{Xhat_r:eq} 
\end{eqnarray} 
where $r=rank(X)$
  assumed known, and $y_1\geq \ldots \geq y_m$. One of the appealing
  properties of TSVD, which helped it gain popularity, is that the estimator
  $\hat{X}_r$ is the best possible approximation of rank $r$ to the data matrix
  $Y$ in the least squares sense \cite{Eckart1936}, and therefore the maximum
  likelihood estimator under Gaussian noise. 
  We note, however, that $\hat{X}_r$ is not necessarily
  the best approximation for the signal matrix $X$, which is in essence more
  relevant~\cite{Donoho2013b}.

\paragraph{Column sparsity.}
The implicit likelihood model solved by TSVD assumes that the entries of $Z$ are
independent and normally distributed, but it does not make any assumptions on
the matrix $X$. In many applications, there is additional information about the
underlying signal matrix $X$ that can be leveraged for estimation. 
Particularly, it is sometimes the case that $X$ is a column-sparse matrix,
meaning that all but $t$ columns of $X$ are zero.   We call the non-zero columns
of $X$ {\em active} columns, and the other columns {\em non-active}.
Equivalently, the right singular vectors of $X$ are all sparse, {\em with a
common same sparsity pattern.} This kind of data emerges in various domains.  In
recommendation systems, estimation of a user-item preference matrix when certain
blocks of users are indifferent to some of the items, or when some columns are
outlier measurements \cite{chen2016matrix}; In signal processing and array
processing, denoising of a signal measured over time, which is either
intermittent or
contaminated by an intermittent interference \cite{besson2002direction}; In
genomics, and specifically in DNA methylation, estimation of strong systematic
confounders poses a key challenge that is well modeled by estimation of
a column-sparse matrix. In Section \ref{meth:sec} below, we focus on 
DNA methylation and provide a detailed, real data example of this application. 

 \paragraph{Connections with Sparse PCA.} 
 Estimation of column-sparse matrices is closely related to sparse PCA.
 There, one is interested in estimating the eigenvectors of $X^\T X$ or $XX^\T$ 
 (assumed sparse), or their support (not assumed to be common for all
 vectors), 
 rather than estimating the matrix $X$ itself;
 this makes sparse PCA a different, and in some sense, 
 a harder problem than the one we address.
 %
 %where the assumption is that each of the singular vectors of $X$ is sparse. 
 Sparse PCA has received considerable attention in
the machine learning and statistics communities, owing to a fascinating
combination of statistical hardness and computational hardness 
\cite{zou2006sparse,moghaddam2006spectral,shen2008sparse,Johnstone2009,Amini2009,Witten2009,cai2013sparse,asteris2014nonnegative,Krauthgamer2015,Deshpande2014}. 
The fundamental limits for consistent support estimation are known
\cite{berthet2013computational,ma2015sum,Deshpande2014}, as are the minimax rates \cite{birnbaum2013minimax,cai2013sparse,vu2013minimax,wang2014tighten}.
%Focusing on algorithms, many of the proposed methods are 
It is natural to ask whether sparse PCA algorithms can be ``lifted'' into
estimators of
column-sparse matrices. Unfortunately, 
many of the algorithms proposed are not computationally feasible for real-life
datasets, or may be difficult to implement or use 
 (e.g., using semidefinite
programming \cite{d2007direct,Amini2009,Krauthgamer2015} or other optimization
techniques \cite{vu2013fantope,cai2013sparse,wang2014tighten}). 
Other methods are 
heuristic in nature in the sense that there are no provable guarantees that
they will provide improved estimates of the eigenvectors (and therefore of $X$),
and indeed it is not known whether 
they outperform even the simplest approach where one applies TSVD. 
%In a few cases
%there are provable guarantees, however these are often non-trivial algorithms
%that are not as efficient as simply running SVD
There are a few exceptions, where simple and computationally efficient 
methods are analyzed and shown to perform better than PCA.  Particularly, \cite{Johnstone2009} proposes a method to detect the
active columns of $X$, and they show that in the limit, their algorithm provides
a consistent estimate of the top singular vector under the assumption that $m/n
\to \beta$, where $\beta > 0$ is a constant.  Moreover, \cite{Amini2009} analyze
the algorithm of \cite{Johnstone2009} and show that it successfully recovers the
$t$ active columns of $X$ if  $t \leq O( \sqrt{m/\log n})$ when $r=1$ and the
singular vector of $X$ has entry $1/\sqrt{t}$ in each of the active
columns.  \cite{Krauthgamer2015} studied covariance thresholding, again for
estimating the top eigenvalues of $XX^T$, yet is not immediately clear how their
methods can be used for direct estimation of a column-sparse $X$.
\paragraph{Motivation: A simple, practical algorithm with theoretical guarantees.}
Our motivation for the algorithm suggested here is based on the gap between practice and theory. 
In practice, most researchers in science that use PCA or SVD do not use the
sparse versions since the algorithms are either too complex or are not
necessarily guaranteeing an improved performance. The algorithm of
\cite{Johnstone2009} is an example for a simple procedure that would be easy to
apply by any practitioner, as it merely computes the norm of each column of the
data matrix $Y$ and then computes SVD on the columns with the largest norms.
Their method, as well as its inherent over-sensitivity to arbitrary scaling of the columns, 
is described in more details in Section \ref{discussion:sec} below. 
%However, their algorithms suffers from an inherent disadvantage in practice, since an arbitrary scaling of the columns where each column is scaled by a different constant is detrimental to their method. 
%Such scaling is unavoidable in practical scenarios. 
Other methods such as \cite{asteris2014nonnegative} assume non-negative entries in the matrix $X$, a problematic
assumption in most practical instances, and particularly in the two example applications above.
%which in theory could be true, 
%however in real data one first needs to 'clean' the data by regressing on confounders such as technical effects, and the resulting signal matrix typically includes negative entries.

\paragraph{The algorithm.}
In this paper we introduce \name{} -- a simple modification of TSVD, 
which is designed to outperform the original on column-sparse data while still
being safe to use even without column-sparsity. 
Our algorithm is extremely simple, so that it is more 
likely to be used correctly by practitioners, 
who are familiar with PCA and SVD, but who may be hesitant  to
adapt more complicated methods.  

For any matrix $X$, let $[X]_j$ denote the $j$-th column of $X$.
Assuming oracle knowledge of the underlying $r$ and the
column-sparsity $t$, the \name{} estimator proceeds in three steps:
\begin{enumerate}
    \item Compute the TSVD $\hat{X}_r$ of the data Y.
    \item Compute the column scalar products
      $c_j=\langle
        [\hat{X}_r]_j,[Y]_j\rangle$ and sort them in absolute value to obtain
        $|c_{j(1)}| \geq |c_{j(2)}| \geq \ldots \geq |c_{j(n)}|$.
    Here, $\left(j(1),\ldots,j(n)\right)$ 
is a permutation of $(1,\ldots,n)$.
    \item Keep the first $t$ columns with largest absolute 
      scalar products,
        namely $[\hat{X}_r]_{j(i)}$ with $1\leq i\leq t$, and set to zero
        the rest. Formally,
        \begin{eqnarray} \label{refactor:eq}
            [\hat{X}^\rf_{r,t}]_{j(i)} = 
            \begin{cases}
                [\hat{X}_r]_{j(i)} & 1\leq i \leq t\\
                0 & t+1\leq i \leq n 
            \end{cases}\,,
          \end{eqnarray}
        where $\hat{X}_{r,t}^\rf$ is the \name{} estimator, with tuning parameters
        $r,t\in \mathbb{N}$.
\end{enumerate}
When $r$ is
understood, we write simply $\hat{X}^\rf_t$. 
%Formally,
%$[\hat{X}_{r,t}^*]_{(i)} = 
%[\hat{X}_r]_{(i)} \cdot\mathbf{1}_{1\leq i \leq t}
%$\,.
%

\paragraph{A preliminary empirical observation.}
Our basic algorithm admits several natural variations. First, we 
can replace the inner products $c_j$ in step 2 with
correlations
between columns
\[
  c^+_j = \frac{\langle [\hat{X}_r]_j,[Y]_j\rangle}
  {||[\hat{X}_r]_j || \cdot\norm{[Y]_j}}\,,
\]
sort them, and let the rest of the algorithm proceed as before.
Importantly, this 
makes the algorithm insensitive to individual column scaling. 
We call  this variation of the algorithm \nameplus{}. Second, 
%we may choose to 
%change the order of column-selection and TSVD in step 3 of \name{}:
%having computed
%$j(1),\dots,j(t)$, 
instead of returning $\hat{X}^{\rf}_{r,t}$
we can return the TSVD of
the matrix $\tilde{Y}$ with
\begin{eqnarray} \label{refactorS:eq}
  [\tilde{Y}]_{j(i)} = 
              \begin{cases}
                [Y]_{j(i)} & 1\leq i \leq t\\
                0 & t+1\leq i \leq n 
              \end{cases}\,.
\end{eqnarray}
Let us call this variation the algorithm \nameS{}.

Recently, it was shown that \nameS{} is extremely efficient in removing strong systematic confounders from 
DNA methylation data \cite{rahmani2016sparse}. The algorithm was presented
there as a heuristic; in
this paper we undertake to analyze its merits formally and explain its success.
It is harder to analyze \nameS{}, but as we show in Section
\ref{simulation:sec} below, its performance is empirically similar to that of \name{} and \nameplus{}. 
%do not significantly alter performance. 
Thus, in this short paper we primarily study the simpler \name{} and \nameplus{} algorithms. 
%,
%as defined in \eqref{refactor:eq}. 

\paragraph{Synopsis.}
Let us measure the performance of an estimator $\hat{X}=\hat{X}(Y)$ for $X$ 
by expected mean square error (MSE). In our case, this is just the Frobenius loss
\[
    \norm{\hat{X}(Y)-X}_F^2\,,
\]
where $\norm{\cdot}_F^2$ is the sum of
squares of matrix entries. The TSVD estimator $\hat{X}_r$ is an optimal rank-$r$
approximation of the data matrix $Y$, in MSE, yet there is no a-priori reason
why it should be a good, or even reasonable, estimator for the signal matrix
$X$. Indeed when $r\ll n$ it was shown to be significantly suboptimal 
\cite{Donoho2013b}. 

In this paper we prove that, when $X$ is low rank and column-sparse, 
the \name{} estimator 
$\hat{X}^\rf_{r,t}$ is as good as the traditional TSVD $\hat{X}$, or better. 
Formally, with high probability,
\[
  \norm{\hat{X}^\rf_{r,t}-X}_F^2 \leq \norm{\hat{X}_r-X}_F^2\,.
\]

In other words, when column-sparsity is known to hold, 
{\em the simple procedure of
removing columns of the TSVD with low correlations with the data matrix is safe
to use}, as it can only
improve estimation. We further prove that the relative improvement
in MSE can be quite substantial. 
In the Supporting Information, we prove analogous results 
for \nameplus{}. (We note that \nameplus{}
  is much more useful in practice, since \name{} 
is sensitive to an arbitrary scaling of the columns.)
%To this end we provide formal guarantees
%on the improvement in MSE offered by \name{}, relative to TSVD,  
%and solid empirical evidence. 
Interestingly, we bring solid empirical evidence that \name{} 
{\em always}  offers improved MSE relative to the
TSVD baseline, regardless of the underlying column sparsity.

\section{Setup and notation} \label{setup:sec}
Column vectors are denoted by boldface letters such as $\V{v}$, 
their transpose by $\V{v}^\T$ and their coordinates e.g.
by $\V{v}=(v_1,\dots, v_m)^\T$. 
Let 
\begin{eqnarray} \label{X:eq}
    X = \sum_{i=1}^r x_i \V{a}_i \V{b}_i^\T
\end{eqnarray}
be a Singular Value Decomposition of the signal matrix $X$ we wish to estimate. 
Here,  $\V{a}_i=\left((a_i)_1,\ldots(a_i)_m\right)^\T\in \R^m$ and
$\V{b}_i=\left( (b_i)_1,\ldots,(b_i)_n \right)\in\R^n$
($i=1,\ldots,r$) are all unit vectors. 
%We assume $\norm{\V{b}_i}=1$ ($1\leq i \leq r$). 
%The vectors $\V{a}_i$ may not
%be normalized, so that \eqref{X:eq} may not be a singular value decomposition of
%$X$; However, $X=\sum_{i=1}^r (x_i \norm{a}_i) (\V{a}_i/\norm{a}_i) \V{b}_i'$ 
%is a singular value decomposition of $X$.
    For column sparsity,  we may reorder the columns if necessary and assume that
    $[X]_j=0$ for $j>t$. This implies 
    $(b_i)_j=0$ for $j>t$ and all $i$. 
%Assuming there is no preferred direction for the right signal singular vectors
%$\V{a}_i$, choose to let $\V{a}_i$ have i.i.d normal entries:
%$(a_i)_j\iid\Nc(0,1)$. This implies that the left 
%singular vectors of $X$ are all independent and uniformly distributed on the
%unit sphere in $\R^m$.
%With $\V{a}_i$ and $\V{b}_i$ so defined, and with the
%signal $X$ as in \eqref{X:eq}, we 
The data matrix available to us is 
\begin{eqnarray*}
    Y= X+(\sigma/\sqrt{n})Z
  \end{eqnarray*}
where $Z$ is an $m$-by-$n$ matrix whose entries are $Z_{i,j}\iid \Nc(0,1)$. 
(This noise normalization is standard in matrix denoising, as it prevents the
  singular values of $Z$ from growing with $n$, keeping a fixed 
signal-to-noise ratio.)
%\TODO{Matan, please write something about this normalization - it looks fishy if you encounter it for the first time}
%\TODO{should assume $\sigma=1$ WLOG at beginning of proofs.}
Throughout this paper, the index $i$ will be used for singular values and
vectors, and the index $j$ will be used for columns. For example, 
$(v_i)_j$ is the $j$-th coordinate of the $i$-th singular vector $\V{v}_i$.

Finally, throughout the paper, 
we will say that an event $A_n$ occurs with high probability 
if $Pr(A_n) = 1-O(\frac{1}{n})$. The parameter $n$ in the context of this paper corresponds to the number of columns of the matrix. Note that if $A_n, B_n$ occur with high probability then $A_n \land B_n$ also occurs with high probability.

\section{Merits of the \name{} statistic} \label{discussion:sec}

Under the ``prior'' that only $t$ columns of $X$ are nonzero, denoising of $X$
is much better done on the active columns alone, namely those columns $j$ 
where $[X]_j\neq 0$. Therefore, a reasonable denoising algorithm will proceed in
two steps: First, detect active columns; Second, denoise using active columns
only, and estimate those columns that were detected to be non-active by $0$. 
A natural method for detecting the active columns is due to
\cite{Johnstone2009}. They considered the simple statistic
\[
    T_j^{\chi} = \norm{[Y]_j}^2 \,.
\]
For non-active columns, this statistic is distributed
$\chi^2_m$, while for active columns it is distributed $\chi^2_m\left(
\norm{[X]_j}^2 \right)$, the latter denoting the non-central $\chi^2$ on $m$
        degrees of freedom with noncentrality parameter $\norm{[X]_j}^2$.
Detection of the active columns would then proceed by testing the hypothesis
that the noncentrality parameter is zero for each column. 

This method does not capitalize on the low-rank assumption. 
To see why, observe that 
\begin{eqnarray*}
    T_j^{\chi} &=&  \langle [Y]_j \, , \, [Y]_j \rangle \\
    &=&  \Big\langle \sum_{i=1}^m y_i \V{u}_i (v_i)_j \, , \, \sum_{k=1}^m y_k
    \V{u}_k (v_k)_j \Big\rangle \\
&=&  \sum_{i,k=1}^m y_i y_k (v_i)_j (v_k)_j \langle \V{u}_i \, , \, \V{u}_k \rangle \\
&=& \sum_{i=1}^m y_i^2 (v_i)_j^2\,.
\end{eqnarray*}
It follows that 
\begin{eqnarray}
    T_j^{\chi} = \sum_{i=1}^r y_i^2 (v_i)_j^2 + \sum_{i=r+1}^m y_i^2
    (v_i)_j^2\,.
    \label{Tchi:eq}
\end{eqnarray}
When $rank(X)=r\ll
m$, the first $r$ right singular vectors hold information regarding $X$, while
all the rest are just noise. The same is true for the singular values. 
Therefore the left sum in \eqref{Tchi:eq}
contains the signal and the right sum contains noise which harms the detection.

In contrast, the \name{} algorithm detects active columns based on the statistic
\[
    T^\rf_j = \langle [Y]_j \, , \, [\hat{X}_r]_j \rangle\,.
\]
The calculation above readily shows that 
\[
    T^\rf_j = \sum_{i=1}^r  y_i^2 (v_i)_j^2 \,,
\]
capturing only the ``signal'' part of $T^\chi_j$.

\section{Main results}
\label{main:sec}

Following \cite{Johnstone2009,Amini2009}, 
we study formally the case $r=1$. As these authors note, this case offers all
the insight of the general case, while allowing proofs to be reasonably readable
and understandable. 
Let $\V{a}\equiv\V{a}_1$ and
$\V{b}\equiv\V{b}_1$, with entries $(a_1,\ldots,a_m)$ and 
$(b_1,\ldots,b_n)$ respectively. 
Similarly let $\V{u}\equiv \V{u}_1$ and $\V{v}\equiv\V{v}_1$ with entries 
$(u_1,\ldots u_m)$ and $(v_1,\ldots,v_n)$. Also write $y\equiv y_1$ for the
leading data singular value and $x\equiv x_1$ for the leading original singular value.
We simplify and write 
$\hat{X}_t^\rf$ for $\hat{X}_{r,t}^\rf$ with $r=1$. As before, 
$\hat{X}_1$ denotes the truncated SVD with $r=1$. Without limiting the
generality of our results, it will also be convenient to
assume a unit noise level $\sigma=1$.

\begin{theorem} \label{main1:thm}
    {\bf \name{} is better when the signal is not too
    weak.}
    Assume that $x>\sqrt{1+2\sqrt{\beta}}$, where $\beta=m/n$. There exists a constant $C$ such
    that if for all $j=1,\ldots, t$ we have 
    \[
      b_j^2> C \frac{\log n}{n}\,,
    \]
    %\TODO{the constant C?} 
    then, with high probability
    %with probability at least $1-O(1/n)$ 
    %\TODO{no constant in front of probability?}
    \[
        \norm{\hat{X}^\rf_t - X}^2_F \leq \norm{\hat{X}_1 - X}^2_F\,.
    \]

\end{theorem}
Interestingly, this theorem does not explicitly assume anything about the sparsity $t$. 
When $t$ is not large, the
condition $ b_j^2> C \log n/n$ is quite mild, since 
there are only $t$ nonzero entries.

In fact, the relative gain in MSE
    offered by $\hat{X}^\rf_t$ with respect to $\hat{X}_1$ is quite massive:
\begin{theorem} \label{gains:thm}
  {\bf Relative improvement in MSE.}
    Make the same assumptions as in Theorem \ref{main1:thm}. 
    For every fixed $\epsilon > 0$, with high probability 
    %at least $1-O(1/n)$,
    we have for the relative improvement in MSE
    \[
\frac{
\norm{ \hat{X}_1 - X}^2_F - \norm{\hat{X}^\rf_t - X}^2_F 
}
{\norm{X}_F^2} \geq 1 - \frac{t+\log n}{n}(1+\epsilon)\,.
    \]
\end{theorem}
%\TODO{the proof requires us to bound $y/x$}
%\TODO{Add this to discussion somewhere: 
 % When $t=n/log(n)$ we still get relative error close to $1$. This hints
%that the proof of theorem 3 should be extended to $t>b/log(n)$.}

Even without any assumption on the signal singular vector, a mild sparsity
assumption is enough to guarantee that \name{} can only improve estimation:
\begin{theorem} \label{main2:thm}
    {\bf Even under mild column sparsity, \name{} is better.}
    Assume that $x> \sqrt{1+2\sqrt{\beta}}$. 
There exists a constant $C_0$ such that if 
\begin{eqnarray} \label{eq:t}
    t\leq C_0\frac{n}{\log n}
  \end{eqnarray} 
  then with high probability
%with probability at least $1-O(1/n)$
    \[
        \norm{\hat{X}^\rf_t - X}^2_F \leq \norm{\hat{X}_1 - X}_F^2\,.
    \]
\end{theorem}
%We note that the condition \eqref{eq:t}
%is {\em tight} in the sense that it matches
%the so-called information theoretic limit for detection of active columns 
%\cite{Amini2009,Krauthgamer2015}. In other words, no algorithm which involves
%detection of active columns, \name{} or other, can succeed if $t$ is significantly
%larger than \eqref{eq:t}. Interestingly though, as Theorem \ref{main2:thm}
% studies estimation of the underlying matrix $X$ rather than of the singular
% vectors, it 
%does not assume a lower bound the entries $b_j$, which is 
%essential to the analysis of \cite{Amini2009,Krauthgamer2015}. 

%The proofs for $r>1$ are similar and does not shed
%new light on the underlying phenomena. 

{\bf Analogous results for \nameplus{}.}
In the Supporting Information
we prove analogous results for \nameplus{},
which we omit here due to space limitations. 

\section{Simulation study} \label{simulation:sec}

We performed a comprehensive simulation study, comparing 
\name{}, \nameS{} and \nameplus{} to the  
baseline TSVD $\hat{X}_r$ of \eqref{Xhat_r:eq}. We also compared two algorithms
based on the method of \cite{Johnstone2009} for detecting ``active'' columns.
Specifically, let $j(1),\dots,j(t)$ be the indices of the $t$ columns of $Y$ 
with
the largest value of $\norm{[Y]_j}^2$. Let \JL{} be the algorithm that uses
\eqref{refactor:eq} to estimate $X$, and let \JLS{} be the algorithm that returns
the TSVD of the matrix as in \eqref{refactorS:eq}.
We chose a fixed noise level $\sigma=1$. In each simulation, we scanned over a
range of values for $x$ (the signal singular value), $n$ (the number of columns,
with number of rows $m$ being held fixed) and column sparsity $t$ (with the
number of columns $n$ being held fixed). These scans were performed for
different values of the underlying $r=rank(X)$ and with the entries of the
noise matrix $Z$ sampled from different noise distributions. 
The full results of our simulation study span some 98 figures, and are shown in
the Supporting Information. 
Sample results are 
shown in Figure
\ref{gaussian_m=200_r=5_sigma=1_00_x=4_00_n=200_noise=gaussian:fig}, Figure
\ref{student-t6_m=200_r=5_sigma=1_00_t=100_n=200_noise=student-t6:fig} and
Figure \ref{low-r_m=200_r=1_sigma=1_00_x=4_00_n=200_noise=gaussian_full:fig}.
%For space considerations, most of
%the simulations are provided in the Supporting Information \TODO{make SI}.
%\paragraph{Reproducible research.}
%The results in this paper are fully reproducible\footnote{We make available the
%  full scripts generating our figures in the {\tt zip} file included with the
%  submission. Following publication, 
%the code will be deposited in a permanent code repository.}.
%\TODO{There was a comment here about "SDR" - I don't know what that is...}.
%The reader can
%experiment with the scripts to generate additional figures in additional
%scenarios omitted for space considerations.

%Try: $x=1.05, x\approx \sqrt{3} = \sqrt{1+2\sqrt{\beta}}\big|_{\beta=1}$. try:
%$x$ degenerate, spread.

\begin{figure}[h] 
\centering
\begin{tikzpicture}
\begin{axis}[
xlabel={$t$, the number of active columns out of $n=200$}, 
ylabel={MSE},
xmin=20, xmax=200,
ylabel near ticks,
legend pos=south east,
]
\addlegendentry{\name{}}
\addplot [color=blue]
 plot[error bars/.cd, y dir = both, y explicit]  table[x = t, y = refactor_mse, y error = refactor_mse_std]{gaussian_m=200_r=5_sigma=1_00_x=4_00_n=200_noise=gaussian.dat};
 
\addlegendentry{TSVD}
\addplot [color=red]
 plot[error bars/.cd, y dir = both, y explicit]  table[x = t, y = tsvd_mse, y error = tsvd_mse_std]{gaussian_m=200_r=5_sigma=1_00_x=4_00_n=200_noise=gaussian.dat};
\addlegendentry{\JL{}}
\addplot [color=green]
 plot[error bars/.cd, y dir = both, y explicit]  table[x = t, y = JL_mse, y error = JL_mse_std]{gaussian_m=200_r=5_sigma=1_00_x=4_00_n=200_noise=gaussian.dat};
\end{axis}
\end{tikzpicture}
\caption {The performance of \name{}, TSVD and \JL{} on a $200\times 200$ matrix of Gaussian noise, with $x=4$, $r=5$. The number of active columns, $t$, is varied, and performance is measured by the MSE of the estimated matrix, averaged across 50 runs. \name{} consistently outperforms the other algorithms across the entire measured range.} 
\label{gaussian_m=200_r=5_sigma=1_00_x=4_00_n=200_noise=gaussian:fig} 
\end{figure}
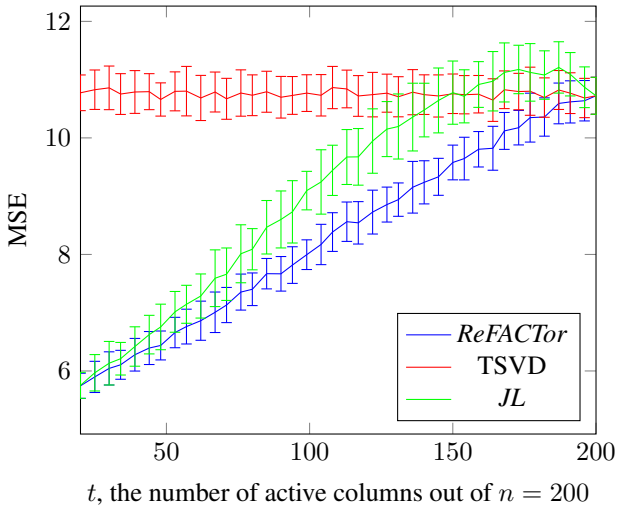 

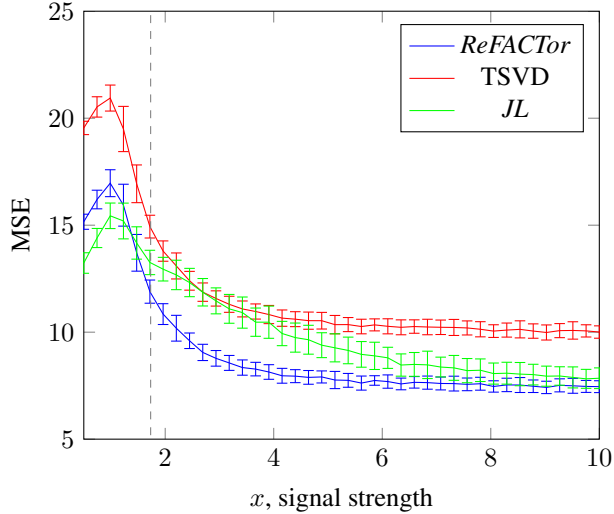
\begin{figure}[h] 
\centering
\begin{tikzpicture}
\begin{axis}[
xlabel={$x$, signal strength}, 
ylabel={MSE},
xmin=0.5, xmax=10,
ymin=5, ymax=25,
ylabel near ticks,
legend pos=north east,
]
\addlegendentry{\name{}}
\addplot [color=blue]
 plot[error bars/.cd, y dir = both, y explicit]  table[x = x, y = refactor_mse, y error = refactor_mse_std]{student-t6_m=200_r=5_sigma=1_00_t=100_n=200_noise=student-t6.dat};
\addlegendentry{TSVD}
\addplot [color=red]
 plot[error bars/.cd, y dir = both, y explicit]  table[x = x, y = tsvd_mse, y error = tsvd_mse_std]{student-t6_m=200_r=5_sigma=1_00_t=100_n=200_noise=student-t6.dat};
\addlegendentry{\JL{}}
\addplot [color=green]
 plot[error bars/.cd, y dir = both, y explicit]  table[x = x, y = JL_mse, y error = JL_mse_std]{student-t6_m=200_r=5_sigma=1_00_t=100_n=200_noise=student-t6.dat};
\addplot[gray, mark=none, dashed] coordinates {(1.732,0) (1.732,30)};
\end{axis}
\end{tikzpicture}
\caption {The performance of \name{}, TSVD and \JL{} on a $200\times 200$ matrix of noise modeled by Student's t distribution with 6 degrees of freedom, with $t=100$ active columns and $r=5$. The signal strength, corresponding to the original singular value $x$, is varied, and performance is measured by MSE, averaged across 50 runs. \name{} outperforms the other algorithms, unless the signal is weak; in particular, it  is better when $x>\sqrt{1+2\sqrt{\beta}}$ (denoted by a dashed line), in line with Theorem~\ref{main2:thm}.}
\label{student-t6_m=200_r=5_sigma=1_00_t=100_n=200_noise=student-t6:fig} 
\end{figure}

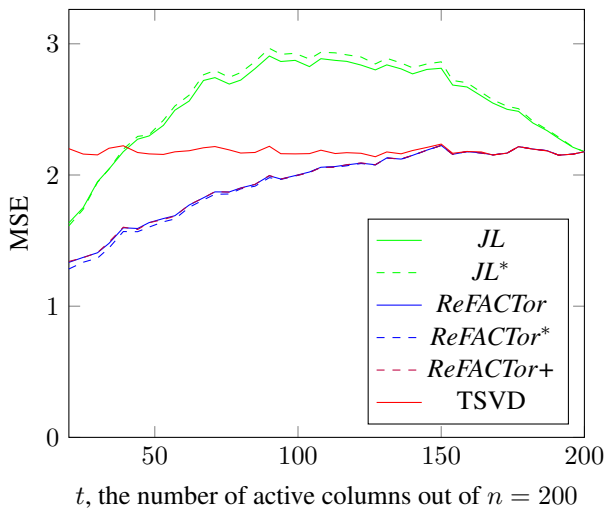
\begin{figure}[h] 
\centering
\begin{tikzpicture}
\begin{axis}[
xlabel={$t$, the number of active columns out of $n=200$}, 
ylabel={MSE},
xmin=20, xmax=200,
ymin=0,
ylabel near ticks,
legend pos=south east,
]
\addlegendentry{\JL{}}
\addplot [color=green] plot table[x = t, y = JL_mse]{low-r_m=200_r=1_sigma=1_00_x=4_00_n=200_noise=gaussian.dat};
\addlegendentry{\JLS{}}
\addplot [color=green, dashed] plot table[x = t, y = JL_mse_full]{low-r_m=200_r=1_sigma=1_00_x=4_00_n=200_noise=gaussian.dat};
\addlegendentry{\name{}}
\addplot [color=blue] plot table[x = t, y = refactor_mse]{low-r_m=200_r=1_sigma=1_00_x=4_00_n=200_noise=gaussian.dat};
\addlegendentry{\nameS{}}
\addplot [color=blue, dashed] plot table[x = t, y = refactor_mse_full]{low-r_m=200_r=1_sigma=1_00_x=4_00_n=200_noise=gaussian.dat};
\addlegendentry{\nameplus{}}
\addplot [color=purple, dashed] plot table[x = t, y = refactor_mse_corr]{low-r_m=200_r=1_sigma=1_00_x=4_00_n=200_noise=gaussian.dat};
\addlegendentry{TSVD}
\addplot [color=red] plot table[x = t, y = tsvd_mse]{low-r_m=200_r=1_sigma=1_00_x=4_00_n=200_noise=gaussian.dat};
\end{axis}
\end{tikzpicture}
\caption {The performance of \name{}, TSVD and \JL{} on a $200\times 200$ matrix of Gaussian noise, with $x=4$ and in a low rank setting, with $r=1$. The number of active columns, $t$, is varied, and performance (MSE) is measured by the Frobenius norm of the difference between the estimated matrix to the original. Standard errors not shown for clarity of presentation. The improvement of \name{} over TSVD and \JL{} is more substantial in a low rank setting. Additionally, the \nameS{}, \nameplus{} and \JLS{} variations do not display a significant difference in performance here relative to their respective counterparts.} 
\label{low-r_m=200_r=1_sigma=1_00_x=4_00_n=200_noise=gaussian_full:fig} 
\end{figure} 

%\paragraph{Analysis.}
Inspection of the empirical evidence suggests the following:\\
  ~\\
{\bf  \name{} is safe to use.} When the true column sparsity is known, and when the singular value $x$ is strong enough (in the sense of Theorems~\ref{main1:thm} and~\ref{main2:thm}), \name{} offers MSE that is
    {\em always} less than or equal, and often noticably 
    smaller than the MSE of the baseline TSVD (Figures \ref{gaussian_m=200_r=5_sigma=1_00_x=4_00_n=200_noise=gaussian:fig}, 
\ref{student-t6_m=200_r=5_sigma=1_00_t=100_n=200_noise=student-t6:fig} and
\ref{low-r_m=200_r=1_sigma=1_00_x=4_00_n=200_noise=gaussian_full:fig}).
For a high $r$ (e.g., $r=40$ for
  $m=n=200, t=100$), the performance of \name{} and \JL{} is nearly identical.
  This shows that even when the low rank assumption on which \name{} capitalizes
  does not hold, it does not suffer from a performance loss.  \\
  ~\\
{\bf \name{} is preferred to \JL{} if $x$ is not weak.}   
    For all values of $x$, except values close to the BBP phase
    transition\footnote{The BBP phase
      transition~\cite{baik2005phase,benaych2012singular} is a phenomenon
      describing the behaviour of the largest singular value of perturbations of
      low rank by large rectangular random matrices. It describes a threshold
      for the largest singular value, depending on matrix size and on the noise
      distribution, under which the unperturbed singular values and vectors
      cannot be estimated. In the context of Gaussian noise, this threshold is
    $\beta^{-1/4}$. }, 
    %\ecomment{Do we really want to have such a long description of BBP? we don't have space!!!} 
    \name{} is preferred to \JL{} (Figure
    \ref{student-t6_m=200_r=5_sigma=1_00_t=100_n=200_noise=student-t6:fig}).
    This is in line with the Discussion above; indeed, when the signal of $x$ is
    weak, it is spread across the entire space, giving \JL{} a slight
    advantage.  \\  
    ~\\
    {\bf Algorithm variations do not matter much.}
    The performance of the \nameS{} (resp. \JLS{}) variant is
  very close to that of \name{} (resp. \JL{}), as seen in
Figure~\ref{low-r_m=200_r=1_sigma=1_00_x=4_00_n=200_noise=gaussian_full:fig}.
Similarly, the performance of \nameplus{} is almost identical to that of \name{}.
%\ecomment{add figures!!}. 
This is particularly advantageous in practice, since both \name{} and JL{} are sensitive to arbitrary scaling of the columns.\\
~\\
{\bf The theoretical requirements are not tight.} 
  The algorithm works well even outside the scope of our theoretical
  results. \\
  ~\\
  %$x>\sqrt{1+2\sqrt{\beta}}$ is not tight
  %(Figure~\ref{low-r_m=200_r=1_sigma=1_00_x=4_00_n=200_noise=gaussian_full:fig}).
  %the requirement $t>??$ is not tight; the requirement $b_j>??$ is not tight
  %\TODO{I am not sure how we can say that kind of assertion if we avoided actual
  %constants in Thms 1 and 3.} 
 {\bf Universality: Results do not depend on the noise distribution.} Performance results do not qualitatively
  change as the distribution of the i.i.d. noise is changed, implying a more
  universal validity to our results.

\section{Real data example: DNA methylation} 
\label{meth:sec}
DNA methylation is the phenomenon whereby a  methyl group is attached to
specific sites in the DNA \cite{robertson2005dna}.  A typical DNA methylation
study generates an $m$-by-$n$ data matrix $Y$, with measurements of DNA
methylation on $m$ subjects at $n$ genomic sites, such that $Y_{i,j}\in [0,1]$
the fraction of cells of individual $i$ that are methylated in position $j$ in
the genome.  In a typical study, the scientist is interested in the correlation
of one or more $m$-by-$1$ disease status vectors $\V{y}$ with each of the
columns of $Y$.  Before interesting correlations with disease vectors can be
detected, the scientist must remove strong systematic confounders from the data.
Modeling the measurements as  $Y=X+\sigma Z$, where $X$ is a low-rank matrix of
strong systematic confounders, the scientist must first form an estimate
$\hat{X}$ of $X$, and use the column space of $\hat{X}$ to test for significant
correlations between $\V{y}$ and the columns of $Y$, after deducting the
contribution of the confounders. 

It was recently shown
\cite{jaffe2014accounting} that a leading source of strong systematic
confounders is {\em cell type composition.}  
Most studies to date have been performed on whole-blood samples; however, blood
is a heterogeneous collection of different cell types, each with a different
typical methylation profile. Indeed, the top left singular vectors of $Y$  have
been shown to be strongly correlated with the cell type composition in
blood~\cite{koestler2013blood}. 
%Consider a vector $\V{y}$ representing a trait
%or a disease status. We are often interested in the correlation between $\V{y}$
%and each of the columns of $Y$. A high correlation indicates a statistical
%association between the trait and the genomic region.  However, i
If a disease status $\V{y}$ is
correlated with the cell type composition (as is the case in many diseases),
then $\V{y}$ will be correlated with columns in which the typical methylation is
different across different cell types. These associations do not indicate a
specific connection between a methylation site and the trait.

If cell type counts were available for each individual, 
one could regress out the influence of cell type composition from each methylation site measurement in order to account for this confounder. In their absence, one could instead regress out the top principal components, which behave as a surrogate for the cell counts.  Recently, it has been shown~\cite{rahmani2016sparse} that sparse PCA results in a much better prediction of the blood cell counts compared to standard PCA. This is mainly because only a subset of methylation sites are differentially methylated across cell types. The algorithm used in~\cite{rahmani2016sparse} is highly similar to the algorithm presented here. However, its use has been heuristic without any theoretical guarantees. This work aims to provide
theoretical guarantees to a slight modification of that algorithm.

To demonstrate the presence of sparse principal components in
methylation data, we used the data of~\cite{liu2013epigenome}, in which a methylation matrix $Y$ of 686 individuals by 103,638 sites is provided. In addition, a boolean phenotype vector $\V{y}$ of length 686 is provided, indicating for each individual if they were diagnosed with Rheumatoid Arthritis (RA). Here, we assume $r=1$; that is, that a single cell type dominates the confounding signal~\cite{rahmani2016sparse}. Let $X$ be the rank-1 matrix whose elements $X_{i,j}$ indicate, for the $i$-th individual and for the $j$-th site, the proportion of the single cell type for the individual, multiplied by the typical difference of methylation level between the dominating cell type and other cell types, for that site~\cite{rahmani2016sparse}.

%We model the effect of cell type composition as follows. Let $C$ be an $m\times r$ matrix, the element $C_{i,j}$ indicating the proportion of a cell type $j$-th on the $i$-th individual. Methylation sites which are variable across cell types are expected to have a significant component in the subspace spanned by the columns of $C$. Thus, the contribution of the cell type composition signal to each methylation site can be expressed by a low rank matrix $X$, whose column space is the column space of $C$, a

Under the assumption that $Y = X + \sigma Z$, where $X$ is a low rank signal matrix with significant column sparsity, if $\V{y}$ is correlated with the left singular vector of $X$ then $\V{y}$ will be correlated with the active columns of $X$. We thus estimate the left singular vector of $X$ using either \nameS{}, TSVD or \JLS{}. (Results using \nameplus{} are similar and are not shown here.) Given such an estimate, we remove from each column of $Y$ its projection on this vector. If the estimate is accurate, the transformed columns of $Y$ will be uncorrelated with $\V{y}$.  We estimate the correlation between each transformed column of $Y$ and $\V{y}$ using logistic regression, which results in a p-value per column, calculated by a standard Wald test. Assuming an accurate estimate, we expect the distribution of the p-values to be approximately uniform, perhaps with very few outliers which might indicate true correlations. As shown in Figure~\ref{meth:fig}, indeed, ReFACTor (with $r = 1$) empirically results in a relatively uniform p-value distribution, while the other methods tend to result in many significant p-values.  A full description of the experiment is available at the Supplementary Information.

%We note that the above experiment should not be interpreted as if there is only one sparse principal component, and no dense component. It merely reflects the fact that a sparse principal component indeed exists, and its effect on $\V{y}$ is larger than the dense principal component. This is reasonable, since it has been previously shown that the blood cell counts in Rheumatoid Arthritis is different than in the healthy population \cite{??}, and that sparse PCA can be used to estimate the cell counts.

\begin{figure}
\centering
%\includegraphics[width=3.3in]{testme}
%\begin{comment} % source
\begin{tikzpicture}
\begin{axis}[
xlabel={ $-\log_{10}(\text{p-value})$, expected}, 
ylabel={ $-\log_{10}(\text{p-value})$, observed},
xmin=0, xmax=6,
ymin=0, ymax=10,
ylabel near ticks,
legend pos=south east,
]
\addlegendentry{\nameS{}}
\addplot[only marks, mark size = 0.5pt, blue] table[x=exp, y=refactor] {faster_RA.dat};
\addlegendentry{TSVD}
\addplot[only marks, mark size = 0.5pt, red] table[x=exp, y=tsvd] {faster_RA.dat};
\addlegendentry{\JLS{}}
\addplot[only marks, mark size = 0.5pt, green] table[x=exp, y=jl] {faster_RA.dat};
\addplot[gray, mark=none] coordinates {(0, 0) (6,6)};
\end{axis}
\end{tikzpicture}
%\end{comment}
\caption {Results of the real data Rheumatoid Arthritis methylation
analysis, presented by quantile-quantile plots
of the $-\log_{10}(p)$-values for the association tests.
Significant deviation from the black line indicates
an inflation arising from a confounder in the
data. Results are shown for \nameS{}, TSVD and \JLS{}. Estimating the effect of cell type composition using \nameS{} results in a significantly lower inflation.} 
\label{meth:fig} 
\end{figure}
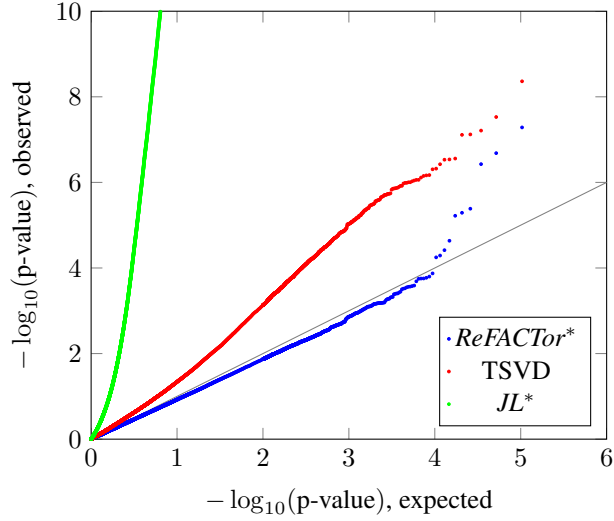

\section{Proof setup and useful Lemmas}

For $r=1$, the detection statistic used by \name{} is simply
\[
    T^\rf_j = y^2 v_j^2\,.
\]
The \name{} algorithm sorts the values $\left\{ T^\rf_j \right\}_{j=1}^n$ and
picks the $t$ columns with the largest value of the detection statistic.
Equivalently, the algorithm sorts the values $\left( (v_1)^2,\ldots,(v_n)^2 \right)$
and picks the $t$ columns with the largest values. 
Let $\V{w}$ denote the projection of $\V{v}$ on $\V{b}^\perp$, so that
\begin{eqnarray}
    \V{v} = c\cdot \V{b} + s\cdot \V{w}\,,
    \label{vperp:eq}
\end{eqnarray}
with 
\[
    c = \langle \V{v},\V{b} \rangle \qquad \qquad s=\|\V{v}-\langle\V{v},\V{b}\rangle\V{b}\|\,.
\]
and $c^2+s^2=1$.
We write $(w_1,\ldots,w_n)$ for the coordinates of $\V{w}$.
%
%\TODO{Maybe it's better to give the bound on $c,s$ here first?}
 
Let us first characterize the distribution of the entries of $\V{w}$.
\begin{lemma}
  {\bf Marginal distribution of the entries of $\V{w}$.}
    Let $w_j$ be the $j$-th entry of $\V{w}$ from \eqref{vperp:eq}, namely the 
    projection of $\V{v}$ on $\V{b}^\perp$. 
    Let $\V{\tilde{w}} \sim \Nc(0,I-\V{b}\V{b}^\T)$. Then
    for $j=1,\dots n$, $w_j$ jointly have the same distribution as
    \[
        \frac{\tilde{w}_j}{\sqrt{\sum_{j=1}^n \tilde{w}_j^2}}
%        \qquad j=1\ldots,n\,
    \]
    and $\sum_{j=1}^n \tilde{w}_j^2\sim \chi^2_{n-1}$.
    \label{w_entries:lem}
\end{lemma}

% TODO: this is not currently needed
% --------
% 
% Marginally, $\tilde{w}_j \sim \Nc(0, 1-b_j^2)$. Note that, neglecting the small dependency that exists between $\tilde{w}_j$ and 
% $\sum_{j=1}^{n}{ \tilde{w}_j^2}$, we have approximately 
% \[
%     w_j \approx
%     \frac{\sqrt{1-b_j^2}}{\sqrt{n-1}}\frac{z_j}{\chi_{n-1}/\sqrt{n-1}}
%     \approx
%     \frac{\sqrt{1-b_j^2}}{\sqrt{n-1}} t_{n-1} 
% \]
% where $z_j\sim\Nc(0,1)$ and $t_{n-1}$ follows Student's-$t$ distribution with
% $n-1$ degrees of freedom. For $n$ large, this implies 
% $w_j\approx \Nc(0,(1-b_j^2)/(n-1))$ and therefore the \name{} algorithm sorts the
% values
% $v_j\approx \Nc(c b_j, s^2(1-b_j^2)/(n-1))$.
% 
% ----------

Toward our main result, we show that the entries of $\V{v}$ on active and
non-active columns differ substantially.

\begin{lemma} {\bf Right singular vector is small in inactive columns.}
  Let $j>t$ and $\alpha>1$. Then 
    \[
        \Prob\left\{ v_j^2 >  \frac{s^2  \alpha^2 \log n}{n} \right\} \leq
        \frac{2}{n^{\alpha^2/2}}\,.
    \]
    \label{inactive:lem}
\end{lemma}

%\begin{lemma} {\bf Right singular vector is not too small in inactive columns.}
%    Let $j>t$. Then 
%    \[
%        \Prob\left\{ v_j^2 <  \frac{s^2  (1-\alpha \sqrt{\log(n)}/2)^2}{n} \right\} \leq
%        \frac{2}{n^{\alpha^2/8}}\,.
%    \]
%    \label{inactive_big:lem}
%\end{lemma}

\begin{lemma}
{\bf Right singular vector is large in active columns.}
Let $j\leq t$ and $\alpha>1$.
%and $\alpha\geq 3$. 
%\TODO{why 3?}
    Assume that 
    \[
    b_j^2 \geq \frac{4 s^2 \alpha^2 \log n}{c^2 n}\,.\] 
    Then
    \[
     \Prob\left\{ v_j^2 < \frac{s^2  \alpha^2  \log n}{n} \right\} \leq
       \frac{2}{n^{\alpha^2/8}}\,.
 \]
    \label{active:lem}
\end{lemma}

We next show that when the signal singular value $x$ is strong enough, the
cosine $c$ from \eqref{vperp:eq}, 
namely the cosine of the angle between $\V{b}$ and $\V{v}$,
is not too small.
\begin{lemma} {\bf A lower bound on the cosine.}
  Let $x> \sqrt{1+2\sqrt{\beta}}$. Then with high probability $c^2 \geq \frac{1}{2}$. 
%   there exists $\varepsilon_0>0$, which depends on $x$ and $\beta$ only,
 % such that 
   % \[
      %  \Pr\left\{ c^2 \leq \frac{1}{2} \right\}
       % <
       % 2 \cdot exp\left( -\varepsilon_0^2 n \beta \right)\,.
   % \]
    \label{cosine:lem}
\end{lemma}

Finally, we show that the singular value in $Y$ is larger than the original singular value, with high probability.
\begin{lemma} {\bf A lower bound on the singular value. } 
Let $x,y$ be defined as above. Then with high probability $y > x$. 
%For a large enough $n$, and for $\gamma > 0$,
%\[
 % \Pr\left\{y < x\right\} < 2n^{-\gamma^2}\,.
%\]
\label{sinval:lem}
\end{lemma}
The following are auxiliary lemmas needed for our main results.
\begin{lemma}
\label{chisquarebound}
 Let $X\sim \chi^2_m$. Then, we have ${\Pr (X \le (1-\epsilon)m) \le e^{-\epsilon^2 m}}$. Also, for every $0 < \epsilon < 2$, we have ${\Pr (X \ge (1+\epsilon)m) \le e^{-\epsilon^2 m/8}}$.
\end{lemma}
\begin{lemma}\label{partialchisquaresum}
Let $w_1,\ldots,w_m \sim \Nc(0,1)$ be independent standard normal random variables variables, and let $w_{(1)},\ldots, w_{(m)}$ be their order statistics. Let $\delta > 0$ be a fixed constant. There is a constant $C > 0$, such  that for $t \leq (1-\delta)m$, with high probability $w_{(1)}^2+\ldots+w_{(m-t)}^2 > C m$. 
\end{lemma}

\section{Proofs}

For space considerations, proofs of Theorem \ref{gains:thm} and all Lemmas 
are deferred to the
Supporting Information.
 
\paragraph{Proof of Theorem \ref{main1:thm}.}
It is easy to see that for $\alpha=4$, Lemmas
\ref{inactive:lem}  and \ref{active:lem}
hold with probability at least $1-O(1/n^2)$ 
for all columns
$j=1,\dots n$. In  addition, by Lemma \ref{cosine:lem}, with high probability $s^2 
\leq c^2$. Thus, 
%Let $\mathcal{E}$ 
%be the event where Lemma \ref{cosine:lem} holds 
%(for $\varepsilon_0$ given by that lemma), 
%\ref{inactive:lem} holds with $\alpha=4$ 
%for $j=1,\cdots,n$ and Lemma \ref{active:lem} holds 
%with $\alpha=4$ 
%for
%$j=1,\cdots,n$.
%
%For large enough $n$
%\mcomment{do we want ``for large enough''? also used in the proof of Lemma
%  \ref{inactive:lem}. But if ``for large enough'', can replace
%any $exp(-cn)$ by $1/n$. looks like we need to decide if for any $n$ or not. if
%``for large enough'' this should be specified in thm statement.}
%, $exp(-\varepsilon_0^2 n \beta/4)<1/2n$, hence
%$\Pr(\mathcal{E})< n(2/n^2+2/n^2) + 2/2n = 5/n$.
%On $\mathcal{E}$, 
 letting $C = 4\alpha^2$, with high probability 
for any $j \leq t$ we have 
$b_j^2 \geq C \log n/n \geq 4s^2\alpha^2\log n/c^2 n$.
By  Lemma \ref{inactive:lem} and Lemma \ref{active:lem},
the value of $v_j^2$ on any active column is
larger than $v_j^2$ on any inactive column with high probability, so that 
\name{} correctly identifies the active columns. In other words, 
$[\hat{X}^\rf_t]_j=[X]_j=0$ on $j>t$, implying 
\begin{eqnarray*} 
    \norm{\hat{X}^\rf_t -X}_F^2 &=& 
    \sum_{j=1}^t \norm{[\hat{X}_1]_j - [X]_j}_F^2 \\
    &\leq& 
    \sum_{j=1}^n \norm{[\hat{X}_1]_j - [X]_j}_F^2 \,
\end{eqnarray*}
as required.
\qed

\paragraph{Proof of Theorem \ref{main2:thm}.}
Let $R^+=\left\{ j(1),\cdots,j(t)\right\}$ denote the set of $t$ indices
detected as active by $T^\rf$ and let $R^-=\left\{ 1,\cdots,n \right\}\setminus
R^+$ denote the indices detected as inactive.
Define
$R^{++}=\left\{ 1,\cdots,t \right\}\cap R^+$ (true positive detections), 
$R^{+-}=\left\{ 1,\cdots,t \right\}\cap R^-$ (false negative detections), 
$R^{-+}=\left\{ t+1,\cdots,n \right\}\cap R^+$ (false positive detections), 
$R^{--}=\left\{ t+1,\cdots,n \right\}\cap R^-$ (true negative detections).
For any set of indices $R$ let 
\[
    \Delta(R) = 
    \sum_{j\in R} \norm{[\hat{X}^\rf_t]_j - [X]_j}_F^2 - 
    \sum_{j\in R} \norm{[\hat{X}_1]_j - [X]_j}_F^2 
\]
denote the gain in MSE over the columns in $R$. 
Clearly $\Delta(R^{++})=\Delta(R^{-+})=0$. It remains to show that $\Delta(R^{+-}) + \Delta(R^{--}) < 0$ with high probability.

 %Now,
%\begin{eqnarray*}
%    \Delta(R^{+-}) + \Delta(R^{--}) &=& 
%    x^2 \sum_{j\in R^{+-}} b_j^2 + y^2 \sum_{j\in R^{--}} v_j^2\\ 
%   & \geq & 
%   -t (x^2 b_j^2) + y^2 \sum_{j\in R^{--}} v_j^2 \,.
%\end{eqnarray*}

First, we bound from below the gain from true negatives. It can be easy seen that $\Delta(R^{--}) = -y^2 \sum_{j\in R^{--}} v_j^2$. Denote by $J$ the set of indices of the $n-2t$ smallest values from $v_{t+1}^2,\ldots,v_{n}^2$, or equivalently, from $\tilde{w}_{t+1}^2,\ldots,\tilde{w}_{n}^2$. Since there are at least $n-2t$ indices in $R^{--}$, and since \name{} detects as inactive the smallest values of the vector $\V{v}$, we have $J\subseteq R^{--}$.  Since  
$t = O(n / \log n)$, by  Lemma~\ref{partialchisquaresum}, there exists $C > 0$ so that with high probability 
\begin{align*}
\sum_{j\in R^{--}}\tilde{w}^2_j \geq C n
\end{align*}
Additionally, by Lemma~\ref{chisquarebound}, with high probability $\|\V{\tilde{w}}\|^2 \leq 2n$. Therefore, since $v_j^2 = s^2\tilde{w}_j^2/\|\V{\tilde{w}}\|^2$, with high probability 
\begin{align*}
\sum_{j\in R^{--}}v^2_j \geq s^2\cdot\frac{C}{2}
\end{align*}
Therefore, by Lemma~\ref{sinval:lem}, with high probability
\begin{align*}
\Delta(R^{--}) = -y^2\cdot \sum_{j\in R^{--}}v^2_j \leq -\frac{x^2 s^2 C}{2}.
\end{align*}

We now bound from above the loss from false negatives. It is easy to verify that $ \Delta(R^{+-}) = x^2 \sum_{j\in R^{+-}} b_j^2$. Let $\alpha = 4$, and denote $T =  (s^2  \alpha^2  \log n) / n$. 
Let $K= \left\{j \mid b_j^2 \geq (4 s^2 \alpha^2 \log n) / (c^2 n) \right\}$.  
By Lemma~\ref{inactive:lem} and using the union bound,  with high probability, for each  $j \in R^{-}$ we have $v_j^2 \leq T$. By Lemma~\ref{active:lem} and using the union bound, for each $j \in K$ with high probability $v_j^2 \geq T$. Thus, with high probability $R^{+-} \cap K = \phi$.
Thus 
\[
  \Delta(R^{+-}) \le |R^{+-}| \frac{4x^2 s^2 \alpha^2 \log n}{c^2 n} \le
\frac{4C_0 x^2 s^2 \alpha^2}{c^2} \,,\]
using $t\le (C_0 n) / \log n$. By Lemma~\ref{cosine:lem}, with high probability $c^2\ge 1/2$, and thus
\begin{align*}
\Delta(R^{+-}) \le 8C_0 x^2 s^2 \alpha^2
\end{align*}
Putting it all together, with high probability and for $C_0$ chosen to be a small enough constant
\[ 
  \Delta(R^{+-})  + \Delta(R^{--}) \leq x^2 s^2 \left( 8C_0 \alpha^2 -
  \frac{C}{2}\right) < 0\,.
\]
\qed

\section{Conclusion}

\name{} is a simple and effective algorithm for the recovery of low-rank
matrices, which are suspected of column-sparsity, in the
presence of noise. \name{} is very simple to implement and indeed is not more
complicated than SVD or PCA. We have proved that \name{} is safe to use, in the
sense that it offers equal or better performance compared to the baseline TSVD
algorithm. Under mild conditions, the performance improvement over TSVD is
provably significant. We have proven similar results for the variant \nameplus{} 
(see Supporting Information). We note that 
the \nameplus{} 
variant is critically important in practice, since it is not affected by scaling of the columns. 
We further presented extensive empirical evidence,
under a very wide variety of conditions,
that \name{}
offers improved performance, sometimes significantly, over
both the baseline TSVD and the algorithm of
\cite{Johnstone2009}. 
 Finally, we have shown that 
\name{} provides scientific value in analysis of DNA methylation studies.

There are numerous important aspects, pertaining to the
theoretical analysis of \name{} as well as to its implementation in the field,
which this brief paper does not cover. For example, the empirical evidence we
present decisively  suggest that our main results hold for 
$r>1$ and $t>n/\log n$; this remains to be shown. Empirical evidence also
decisively suggest that the performance of \name{} is very similar to that of 
\nameS{}; this also remains to be formally analyzed.
Importantly, space did not allow us to discuss estimation of the parameters
$\sigma$, $t$ and $r$, all of which are needed in order to successfully implement 
\name{} in practice.

\section*{Acknowledgements} 

The authors would like to thank Boaz Nadler and Elad Hazan for valuable feedback. This work was partially supported by the Edmond J. Safra Center for Bioinformatics at Tel Aviv University. M.G. was partially supported by German-Israeli foundation for scientific research and development program no. I-1100-407.1-2015, Israeli Science Foundation grant no. 1523/16. E.R. was partially supported by the Israel Science Foundation (Grant 1425/13) and by Len Blavatnik and the Blavatnik Research Foundation. R.S. is supported by the Colton Family Foundation. 

\bibliographystyle{unsrtnat} 
\bibliography{refactor_icml}
%\bibliographystyle{icml2017}

\begin{comment}
\clearpage
\section*{TODOs}
\begin{enumerate}
  \item Matan: lemma \ref{inactive:lem}, lemma \ref{active:lem}, lemma
    \ref{cosine:lem} and theorem 1 done. Missing thm 2, thm3. need to decide on
    constants, ``large enough $n$'' (see comment), etc.
  \item empirical comparison to Johnstone Lu?
\end{enumerate}

\section*{Extensions and future work}
\begin{enumerate}
    \item Estimate the column sparsity $t$
    \item Use optimal shrinkage instead of TSVD
    \item Package the previous two together to produce an algorithm that does
        not need oracle information (can work out of the box on $Y$ without
            knowledge of $r$ and $t$
    \item Proof of main results for $r>1$
    \item SVD-after-refactor: so we improve by running TSVD (or optimal
        shrinakge) on the selected columns - but taken from $Y$ not
        $\hat{X}_r$ (namely on $[Y]_{(i)}$, $1\leq i \leq t$.
    \item Johnstone-Lu and false positive / false negative rate
\end{enumerate}
\end{comment}

\newpage

\end{document}